\documentclass{article}

\usepackage{microtype}
\usepackage{graphicx}
\usepackage{subfigure}
\usepackage{booktabs} 

\usepackage{hyperref}

\usepackage[accepted]{icml2019}

\usepackage{amsmath}
\usepackage{amsfonts}
\usepackage{amssymb}
\usepackage{amsthm}

\usepackage{chngcntr}

\newtheorem{theorem}{Theorem}
\counterwithin*{theorem}{section} 
\newtheorem{lemma}{Lemma}

\newtheorem{remark}{Remark}

\usepackage{verbatim}
\usepackage{algorithm}
\usepackage{algorithmicx}
\usepackage{algpseudocode}
\usepackage{math_defs}

\icmltitlerunning{Low-Complexity Data-Parallel Earth Mover's Distance Approximations}

\usepackage{times}  

\begin{document}

\twocolumn[
\icmltitle{Low-Complexity Data-Parallel Earth Mover's Distance Approximations}

\begin{icmlauthorlist}
\icmlauthor{Kubilay Atasu}{ibm}
\icmlauthor{Thomas Mittelholzer}{hsr}
\end{icmlauthorlist}

\icmlaffiliation{ibm}{IBM Research - Zurich}
\icmlaffiliation{hsr}{HSR Hochschule f\"{u}r Technik, Rapperswil}
\icmlcorrespondingauthor{Kubilay Atasu}{kat@zurich.ibm.com}

\icmlkeywords{Similarity Search, Wasserstein Distance, Earth Mover's Distance, Word Mover's Distance}

\vskip 0.3in
]

\printAffiliationsAndNotice{}  

\begin{abstract}

The Earth Mover's Distance (EMD) is a state-of-the art metric for comparing discrete probability distributions, but its high distinguishability comes at a high cost in computational complexity. Even though linear-complexity approximation algorithms have been proposed to improve its scalability, these algorithms are either limited to vector spaces with only a few dimensions or they become ineffective when the degree of overlap between the probability distributions is high. We propose novel approximation algorithms that overcome both of these limitations, yet still achieve linear time complexity. All our algorithms are data parallel, and thus, we take advantage of massively parallel computing engines, such as Graphics Processing Units (GPUs). On the popular text-based 20 Newsgroups dataset, the new algorithms are four orders of magnitude faster than a multi-threaded CPU implementation of Word Mover's Distance and match its nearest-neighbors-search accuracy. On MNIST images, the new algorithms are four orders of magnitude faster than a GPU implementation of the Sinkhorn's algorithm while offering a slightly higher nearest-neighbors-search accuracy.

\end{abstract}

\section{Introduction}

Earth Mover’s Distance (EMD) was initially proposed in the image retrieval field to quantify the similarity between images~\cite{rubnertg98}. In the optimization theory, a more general formulation of EMD, called Wasserstein distance, has been used extensively to measure the distance between probability distributions~\cite{villani2003}. In statistics, an equivalent measure is known as Mallows distance~\cite{emd_mallows}. 
This paper uses the EMD measure for similary search in image and text databases. 

In the text retrieval domain, an adaptation of EMD, called Word Mover’s Distance (WMD), has emerged as a state-of-the-art semantic similarity metric~\cite{kusnerskw15}. 
WMD captures semantic similarity by using the concept of word embeddings in the computation of EMD. Word embeddings map words into a high-dimensional vector space such that the words that are semantically similar are close to each other. These vectors can be pre-trained in an unsupervised way, e.g., by running the Word2vec algorithm~\cite{mikolovcorr2013} on publicly available data sets. The net effect is that, given two sentences that cover the same topic, but have no words in common, traditional methods, such as cosine similarity, fail to detect the similarity. 
However, WMD detects and quantifies the similarity by taking the proximity between different words into account.

What makes EMD-based approaches attractive is their high search and classification accuracy. However, such an accuracy does not come for free. 
In general, the time complexity of computing these measures grows cubically in the size of the input probability distributions. 
Such a high complexity renders their use impractical for large datasets. Thus, there is a need for low-complexity approximation methods.

EMD can be computed in quadratic time complexity when an $L_1$ ground distance is used~\cite{Ling2007,Gudmundsson07}. 
In addition, approximations of EMD can be computed in linear time by embedding EMD into Euclidean space~\cite{Indyk2003}. 
However, such embeddings result in high distortions in high-dimensional spaces~\cite{Naor2006}. 
An algorithm for computing EMD in the wavelet domain has also been proposed~\cite{ShirdhonkarJ08}, which achieves linear time complexity 
in the size of the input distributions. However, the complexity grows exponentially in the dimensionality of the underlying vector space. 
Thus, both linear-complexity approaches are impractical when the number of dimensions is more than three or four. 
For instance, they are not applicable to WMD because the word vectors typically have several hundred dimensions.

A linear-complexity algorithm for computing approximate EMD distances over high-dimensional vector spaces has also been proposed~\cite{AtasuBigData17}. 
The algorithm, called Linear-Complexity Relaxed Word Mover's Distance (LC-RWMD), 
achieves four orders of magnitude improvement in speed with respect to WMD. 
In addition, on compact and curated text documents, it computes high-quality search results that are comparable to those found by WMD.
Despite its scalability, the limitations of LC-RWMD are not well understood. Our analysis shows that 1) it is not applicable to dense histograms, and 2) its accuracy decreases when comparing probability distributions with many overlapping coordinates. Our main contributions are as follows: 
\begin{itemize}
\itemsep0em 
\item We propose new distance measures that are more robust and provably more accurate than LC-RWMD.
\item We show that the new measures effectively quantify the similarity between dense as well as overlapping probability distributions, e.g., greyscale images.
\item We show that the new measures can be computed in linear time complexity in the size of the input probability distributions: the same as for LC-RWMD.
\item We propose data-parallel algorithms that achieve four orders of magnitude speed-up with respect to state of the art without giving up any search accuracy.
\end{itemize}

\section{Background} \label{sec:background}

EMD can be considered as the discrete version of the Wasserstein distance, and can be used to quantify
the affinity between discrete probability distributions. Each probability distribution is modelled as a histogram,
wherein each bin is associated with a weight and a coordinate in a multi-dimensional vector 
space. For instance, when measuring the distance between greyscale images, the histogram weights are given by
the pixel values and the coordinates are defined by the respective pixel positions (see Fig.~\ref{fig:illustration} (a)).  


The distance between two histograms is calculated as the cost of transforming one into the other. 
Transforming a first histogram into a second one involves moving weights from the bins of the first
histogram into the bins of the second, thereby constructing the second histogram from the first. 
The goal is to minimize the total distance travelled, wherein the pairwise distances between different 
histogram bins are computed based on their respective coordinates. This optimization problem is 
well studied in transportation theory and is the discrete formulation of the Wasserstein distance. 

\begin{figure}[tb!]
  \centering%
  \includegraphics*[width=.8\linewidth]{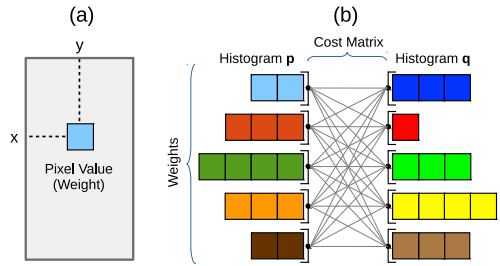}
  \caption{(a) Converting a 28x28 image into a histogram with $h$=28x28=784 bins. The weights are the pixel values and the embedding vectors are the pixel coordinates. 
  (b) Computing the EMD between two flattened histograms for an $h$x$h$ cost matrix $\vect{C}$.}
  \vspace{-0.05in}
  \label{fig:illustration}
\end{figure}

Assume that histograms $\vect{p}$ and $\vect{q}$ are being compared, where $\vect{p}$ has $h_p$ entries and $\vect{q}$ has $h_q$ entries. 
Assume also that an $h_p \times h_q$ nonnegative cost matrix $\vect{C}$ is available.
Note that ${p}_i$ indicates the weight stored in the $i$th bin 
of histogram $\vect{p}$, ${q}_j$ the weight stored in the $j$th bin of histogram $\vect{q}$, and ${C}_{i,j}$ the distance between the coordinates of the $i$th bin of $\vect{p}$ and the $j$th bin of $\vect{q}$ (see Fig.~\ref{fig:illustration} (b)). 
Suppose that the histograms are $L_1$-normalized: $\sum_i {p}_i = \sum_j {q}_j = 1$. 

We would like to discover a non-negative flow matrix $\vect{F}$, where ${F}_{i,j}$ indicates how much of the bin $i$ of $\vect{p}$ has to \textit{flow} to the bin $j$ of $\vect{q}$, such that the cost of moving $\vect{p}$ into $\vect{q}$ is minimized. Formally, the objective of EMD is as follows:
\begin{equation}
\text{EMD}(\vect{p},\vect{q}) = \min_{{F}_{i,j} \geq 0} \sum_{i,j} {F}_{i,j}\cdot{C}_{i,j} . 
\label{eq:objective}
\end{equation}


A valid solution to EMD has to satisfy the so-called \textit{out-flow} (\ref{eq:outflow}) and \textit{in-flow} (\ref{eq:inflow}) constraints.
The \textit{out-flow} constraints ensure that, for each $i$ of $\vect{p}$, the sum of all the flows exiting $i$ is equal to ${p}_i$. 
The \textit{in-flow} constraints ensure that, for each $j$ of $\vect{q}$, the sum of all the flows entering $j$ is equal to ${q}_j$. 
These constraints guarantee that all the mass stored in $\vect{p}$ is transferred and $\vect{q}$ is reconstructed as a result. 
\begin{equation}
\sum_j {F}_{i,j} = {p}_i
\label{eq:outflow}
\end{equation}
\begin{equation}
\sum_i {F}_{i,j} = {q}_j
\label{eq:inflow}
\end{equation}

Computation of EMD requires solution of a minimum-cost-flow problem on a bi-partite graph, wherein 
the bins of histogram $\vect{p}$ are the source nodes, the bins of histogram $\vect{q}$ are the sink nodes, and
the edges between the source and sink nodes indicate the pairwise transportation costs. Solving 
this problem optimally takes supercubical time complexity in the size of the input histograms~\cite{Ahuja1993}.

\subsection{Relaxed Word Mover's Distance}~\label{sec:rwmd}


To reduce the complexity, an approximation algorithm called Relaxed Word Mover's Distance (RWMD) was proposed~\cite{kusnerskw15}. 
RWMD computation involves derivation of two asymmetric distances. 
First, the \textit{in-flow} constraints are relaxed and the relaxed problem is solved using only \textit{out-flow} constraints.
The solution to the first relaxed problem is a lower bound of EMD.
After that, the \textit{out-flow} constraints are relaxed and a second relaxed optimization problem is solved using only \textit{in-flow} constraints, which computes a second lower bound of EMD.
RWMD is the maximum of these two lower bounds. Therefore, it is at least as tight as each one. In addition, it is symmetric. 

Finding an optimal solution to RWMD involves mapping the coordinates of one histogram to the closest coordinates of the other. 
Just like EMD, RWMD requires a cost matrix $\vect{C}$ that stores the pairwise distances between coordinates of $\vect{p}$ and $\vect{q}$. 
Finding the closest coordinates corresponds to row-wise and column-wise minimum operations in the cost matrix (see Fig.~\ref{fig:rwmd}).
To compute the first lower bound, it is sufficient to find the column-wise minimums in the cost matrix, and then 
perform a dot product with the weights stored in $\vect{p}$. Similarly, to compute the second lower bound, it is sufficient 
to find the row-wise minimums and then perform a dot product with the weights stored in $\vect{q}$.
The complexity of RWMD is given by the cost of constructing the cost matrix $\vect{C}$: it requires quadratic 
time and space in the size of the input histograms. Computing the row-wise and column-wise minimums of $\vect{C}$ also has quadratic time complexity.

\begin{figure}[tb!]
  \centering%
  \includegraphics*[width=.5\linewidth]{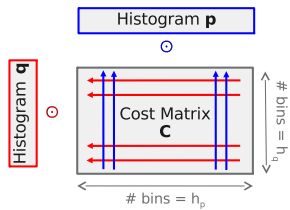}
  \vspace{-0.1in}
  \caption{Quadratic-complexity RWMD computation}
 \label{fig:rwmd}
  \vspace{-0.2in}
\end{figure}




\subsection{Linear-Complexity RWMD (LC-RWMD)}

When computing RWMD between only two histograms, it is not possible to avoid a quadratic time complexity. 
However, in a typical information retrieval system, a query histogram is compared with a large database 
of histograms to identify the top-$\ell$ most similar histograms in the database. 
It is shown in~\cite{AtasuBigData17} that the RWMD computation involves redundant and repetitive
operations in such cases, and that eliminating this redundancy reduces the average time complexity from quadratic to linear. 

Assume that a query histogram is being compared with two database histograms. Assume also
that the two database histograms have common coordinates. A simple replication of the RWMD computation
would involve creation of two cost matrices with identical rows for the common coordinates. 
Afterwards, it would be necessary to perform reduction operations on these identical rows to compute the row-wise minimums. 
It is shown in~\cite{AtasuBigData17} that both of these redundant operations can be eliminated by 
1) constructing a vocabulary that stores the union of the coordinates that occur in the database histograms, 
and 2) computing the minimum distances between the coordinates of the vocabulary and the coordinates of the query only once.

Table~\ref{tab:parameters} lists the algorithmic parameters that influence the 
complexity. Table~\ref{tab:complexity} shows the complexity of RWMD and LC-RWMD 
algorithms when comparing one query histogram with $n$ database histograms. 
When the number of database histograms ($n$) is in the order of the size of the vocabulary 
($v$), the LC-RWMD algorithm reduces the complexity by a factor of the average histogram 
size ($h$). Therefore, whereas the time complexity of a brute-force RWMD implementation 
scales quadratically in the histogram size, the time complexity of LC-RWMD scales only linearly. 

\begin{table}[t!]
\caption{Algorithmic Parameters}
\label{tab:parameters}
\begin{center}
\begin{tabular}{lc}
  \hline
  $n$           & Number of database histograms\\
  $v$           & Size of the vocabulary\\ 
  $m$           & Dimensionality of the vectors \\ 
  $h$           & Average histogram size \\ 
  \hline
\vspace{-0.4in}
\end{tabular}
\end{center}
\end{table}

\begin{table}[t!]
\caption{Complexity of Computing $n$ RWMD Distances}
\label{tab:complexity}
\begin{center}
\begin{tabular}{ccc}
  \hline
              & Time Complexity & Space Complexity \\ 
  \hline
  LC-RWMD     & $O(vhm+nh)$        & $O(nh+vm+vh)$ \\
  RWMD        & $O(nh^2m)$         & $O(nhm)$ \\  
  \hline
  \vspace{-0.4in}
\end{tabular}
\end{center}
\end{table}


\section{Related Work}


A regularized version of the optimal transport problem can be solved more efficiently than network-flow-based approaches~\cite{Cuturi13}. 
The solution algorithm is based on Sinkhorn's matrix scaling technique~\cite{Sinkhorn}, and thus, it is called Sinkhorn's algorithm. 
Convolutional implementations can be used to reduce the time complexity of Sinkhorn's algorithm~\cite{Solomon2015}, e.g., when operating on images.
Given an error term $\epsilon$, the time complexity of Sinkhorn's algorithm is $O((h^2\log{h})/\epsilon^3)$ when computing the distance between histograms of size $O(h)$~\cite{Altschuler2017}.
In addition, a cost matrix has to be constructed, which incurs an additional complexity of $O(h^2m)$ when using $m$-dimensional coordinates.

Several other lower bounds of EMD have been proposed~\cite{assent2008efficient,xu2010efficient,ruttenberg2011indexing,wichterich2008efficient,xu2016emd,huang2016heads,huang2014melody}.
These lower bounds are typically used to speed-up the EMD computation based on pruning techniques.
Alternatively, EMD can be computed approximately using a compressed representation~\cite{Uysal2016approximation,Pele2009}. 

A greedy network-flow-based approximation algorithm has also been proposed~\cite{Gottschlich2014}, 
which does not relax the \textit{in-flow} or \textit{out-flow} constraints. Therefore, it is not a data-parallel algorithm 
and its complexity is quadratic in the histogram size. In addition, it produces an upper bound rather than a lower bound of EMD. 

\section{New Relaxation Algorithms}


In this section, we describe improved relaxation algorithms that address the weaknesses of the RWMD measure and its linear-complexity implementation (LC-RWMD). 
Assume that we are measuring the distance between two histograms $\vect{p}$ and $\vect{q}$. 
Assume also that the coordinates of the two histograms fully overlap but the respective weights are different (see Fig.~\ref{fig:same_coordinates}). 
In other words, for each coordinate $i$ of $\vect{p}$, there is an identical coordinate $j$ of $q$, for which $\vect{C}_{i,j}=0$. 
Therefore, RWMD estimates the total cost of moving $\vect{p}$ into $\vect{q}$ and vice versa as zero even though $p$ and $q$ are 
not the same. This condition arises, for instance, when we are dealing with dense histograms. In other cases, the data of interest 
might actually be sparse, but some background noise might also be present, which results in denser histograms. 

\begin{figure}[tb!]
  \centering%
  \includegraphics*[width=0.8\linewidth]{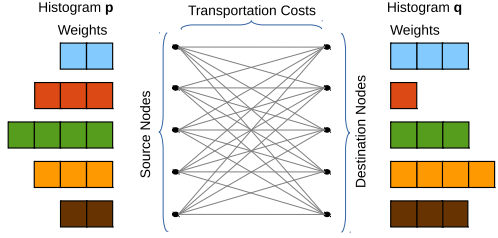}
  \vspace{-0.15in}
  \caption{Different histograms with identical coordinates}
  \vspace{-0.15in}
  \label{fig:same_coordinates}
\end{figure}

In general, the more overlaps there are between the coordinates of $\vect{p}$ and $\vect{q}$, the higher 
the approximation error of RWMD is. The main reason for the error is that when coordinate 
$i$ of $\vect{p}$ overlaps with coordinate $j$ of $\vect{q}$, RWMD does not take into account the fact that the
respective weights $p_i$ and $q_j$ can be different. In an optimal solution, we would not be moving
a mass larger than the minimum of $p_i$ and $q_j$ between these two coordinates. This is a fundamental 
insight that we use in the improved solutions we propose.

Given $\vect{p}$, $\vect{q}$ and $\vect{C}$, our goal is to define new distance measures that relax fewer EMD constraints than RWMD,
and therefore, produce tighter lower bounds on EMD. Two asymmetric distances can be computed by deriving 1) the cost of moving 
$\vect{p}$ into $\vect{q}$ and 2) the cost of moving $\vect{q}$ into $\vect{p}$. If both are lower bounds on $\text{EMD}(\vect{p},\vect{q})$, 
a symmetric lower bound can be derived, e.g., by using the maximum of the two. Thus, we consider only the computation of the cost of moving $\vect{p}$ into $\vect{q}$ without loss of essential generality.

When computing the cost of moving $\vect{p}$ to $\vect{q}$ using RWMD, the \textit{in-flow} constraints of (\ref{eq:inflow}) are removed. 
In other words, all the mass is transferred from $\vect{p}$ to the coordinates of $\vect{q}$, but the resulting distribution is not the same as $\vect{q}$.
Therefore, the cost of transforming $\vect{p}$ to $\vect{q}$ is underestimated by RWMD. To achieve better approximations of $\text{EMD}(\vect{p},\vect{q})$, instead of 
removing the \textit{in-flow} constraints completely, we propose the use of a relaxed version of these constraints: 
\begin{equation}
F_{i,j} \leq {q}_j \quad \text{for all $i, j$} .
\label{eq:relaxed}
\end{equation} 
The new constraint ensures that the amount of weight that can be moved from a coordinate $i$ of $\vect{p}$ to a coordinate $j$ of $\vect{q}$ cannot 
exceed the weight ${q}_j$ at coordinate $j$. However, even if (\ref{eq:relaxed}) is satisfied, the total weight moved to coordinate $j$ of $\vect{q}$ from all the coordinates of $\vect{p}$ can exceed ${q}_j$, potentially violating (\ref{eq:inflow}). Namely,  (\ref{eq:inflow}) implies (\ref{eq:relaxed}), but not vice versa.

When (\ref{eq:relaxed}) is used in combination with (\ref{eq:outflow}), we have:
\begin{equation}
F_{i,j} \leq \min{(p_i,q_j)} \quad \text{for all $i, j$} .
\label{eq:relaxed2}
\end{equation} 
Note that we are essentially imposing capacity constraints on the edges of the flow network (see Fig.\ref{fig:example}) based on (\ref{eq:relaxed2}).

\begin{figure}[tb!]
  \centering%
  \includegraphics*[width=0.80\linewidth]{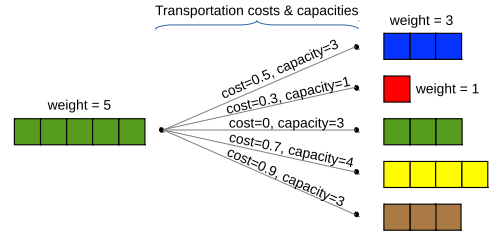}
  \vspace{-0.15in}
  \caption{Imposing capacity constraints on the edges}
  \vspace{-0.15in}
  \label{fig:example}
\end{figure}



We would like to stress that in the framework of this work, which considers the discrete and not the 
continuous case of Wasserstein distances, the only requirement on the cost matrix is that it is nonnegative. 
Since any nonnegative cost $c$ between two locations can be written as the $p$-th power of the $p$-th 
root of $c$ for $p \geq 1$, one can assume that we are dealing with a $p$-th Wasserstein distance~\cite{Villani}.

In the following subsections, we describe three new approximation methods. 
The Overlapping Mass Reduction (OMR) method imposes the relaxed constraint (\ref{eq:relaxed}) only between overlapping coordinates, and is the lowest-complexity and the least accurate approximation method.
The Iterative Constrained Transfers (ICT) method imposes constraint (\ref{eq:relaxed}) between all coordinates of $\vect{p}$ and $\vect{q}$, and is the most complex and most accurate approximation method.
The Approximate Iterative Constrained Transfers (ACT) method imposes constraint (\ref{eq:relaxed}) incrementally between coordinates of $\vect{p}$ and $\vect{q}$, and is an approximation of the ICT method.
Therefore, both its complexity and its accuracy are higher than those of OMR, but lower than those of ICT.

\subsection{Overlapping Mass Reduction}

The OMR method imposes (\ref{eq:relaxed}) only between overlapping coordinates.
The main intuition behind OMR method is that if the coordinate $i$ of $\vect{p}$ and the coordinate $j$ of $\vect{q}$ overlap (i.e., ${C}_{i,j}=0$), 
a transfer of $\min({p}_i,{q}_j)$ can take place free of cost between $\vect{p}$ of $\vect{q}$. After that, the remaining weight in $p_i$ is transferred simply to the second closest coordinate in $\vect{q}$ 
as this is the next least costly move. Therefore, the method computes only the top-2 smallest values in each row of $\vect{C}$.
A detailed description is given in Algorithm~\ref{alg:omr}.


\begin{algorithm}
\begin{algorithmic}[1]
	\Function{OMR}{$\vect{p}, \vect{q}, \vect{C}$}
	\State $t = 0$\ \Comment{initialize transportation cost $t$} 
	\For{$i=1\ldots,h_p$}\ \Comment{iterate the indices of $\vect{p}$}
                \State $\vect{s} =\text{argmin}_{2}({C}_{i,[1 \ldots h_q]})$ \Comment{find top-$2$ smallest}
		\If{${C}_{i,\vect{s}[1]}==0$}\ \Comment{if the smallest value is 0}
			\State{$r = \min({p}_i, {q}_{\vect{s}[1]})$}\ \Comment{size of max. transfer} 
			\State{${p}_i= {p}_i - r$}\ \Comment{move $r$ units of ${p}_i$ to ${q}_{\vect{s}[1]}$} 
			\State{$t = t + {p}_i \cdot {C}_{i,\vect{s}[2]}$}\ \Comment{move the rest to ${q}_{\vect{s}[2]}$}  
		\Else
			\State{$t = t + {p}_i \cdot {C}_{i,\vect{s}[1]}$}\ \Comment{move all of ${p}_i$ to ${q}_{\vect{s}[1]}$}  
		\EndIf
	\EndFor
	\State \Return $t$\ \Comment{return transportation cost $t$}
	\EndFunction
\end{algorithmic}
\caption{Optimal Computation of OMR}
\label{alg:omr}
\end{algorithm}


\subsection{Iterative Constrained Transfers}

The ICT method imposes the constraint (\ref{eq:relaxed}) between all coordinates of $\vect{p}$ and $\vect{q}$. 
The main intuition behind the ICT method is that because the inflow constraint (\ref{eq:inflow}) is relaxed, the optimal flow exiting each source node can be determined independently.
For each source node, finding the optimal flow involves sorting the destination nodes in the ascending order of transportation costs, and then performing iterative mass transfers 
between the source node and the sorted destination nodes under the capacity constraints (\ref{eq:relaxed}). Algorithm~\ref{alg:ict} describes the ICT method in full detail.


\begin{algorithm}
\begin{algorithmic}[1]
	\Function{ICT}{$\vect{p}, \vect{q}, \vect{C}$}
	\State $t = 0$\ \Comment{initialize transportation cost $t$} 
	\For{$i=1\ldots,h_p$}\ \Comment{iterate the indices of $\vect{p}$}
		\State $\vect{s}=\text{argsort}({C}_{i,[1 \ldots h_q]})$ \Comment{sort indices by value}
		\State $l=1$\ \Comment{initialize $l$}
		\While{${p}_i > 0$}\ \Comment{while there is mass in ${p}_i$}
			\State{$r = \min({p}_i,{q}_{\vect{s}[l]})$}\ \Comment{size of max. transfer} 
			\State{${p}_i= {p}_i - r$}\ \Comment{move $r$ units of ${p}_i$ to ${q}_{\vect{s}[l]}$} 
			\State{$t = t + r \cdot {C}_{i,\vect{s}[l]}$}\ \Comment{update cost}  
			\State{$l=l+1$}\ \Comment{increment $l$} 
		\EndWhile
	\EndFor
	\State \Return $t$\ \Comment{return transportation cost $t$}
	\EndFunction
\end{algorithmic}
\caption{Optimal Computation of ICT}
\label{alg:ict}
\end{algorithm}
\vspace{-0.05in}

Algorithm~\ref{alg:aict} describes an approximate solution to ICT (ACT), which offers the possibility to terminate the ICT iterations before all the mass is transferred from $\vect{p}$ to $\vect{q}$.
After performing a predefined number $k-1$ of ICT iterations, the mass remaining in $\vect{p}$ is transferred to the $k$-th closest coordinates of $\vect{q}$, making the solution approximate. 


Theorem~\ref{thm:ict} establishes the optimality of Algorithm~\ref{alg:ict}. Theorem~\ref{theorem:ordering} establishes the relationship between different distance measures. 
The proofs and the derivation of the complexity of the algorithms are omitted for brevity.

\begin{theorem} \label{thm:ict}
  (i) The flow $F^*$ of Algorithm~\ref{alg:ict} is an optimal solution of the relaxed minimization problem given by (\ref{eq:objective}), (\ref{eq:outflow}) and (\ref{eq:relaxed}).
  (ii) ICT provides a lower bound on EMD.
\end{theorem}

\begin{theorem} \label{theorem:ordering}
  For two normalized histograms $\vect p$ and $\vect q$:
  $\text{RWMD}(\vect{p},\vect{q}) \leq  \text{OMR}(\vect{p},\vect{q}) \leq \text{ACT}(\vect{p},\vect{q}) \leq \text{ICT}(\vect{p},\vect{q}) \leq \text{EMD}(\vect{p},\vect{q})$.
\end{theorem}

\begin{algorithm}
\begin{algorithmic}[1]
	\Function{ACT}{$\vect{p}, \vect{q}, \vect{C}, k$}
	\State $t = 0$\ \Comment{initialize transportation cost $t$} 
	\For{$i=1\ldots,h_p$}\ \Comment{iterate the indices of $\vect{p}$}
		\State $\vect{s}=\text{argmin}_{k}({C}_{i,[1 \ldots h_q]})$ \Comment{find top-$k$ smallest}
		\State $l=1$\ \Comment{initialize $l$}
		\While{$l < k$}\
			\State{$r = \min({p}_i, {q}_{\vect{s}[l]})$}\ \Comment{size of max. transfer} 
			\State{${p}_i= {p}_i - r$}\ \Comment{move $r$ units of ${p}_i$ to ${q}_{\vect{s}[l]}$} 
			\State{$t = t + r \cdot {C}_{i,j}$}\ \Comment{update cost}  
			\State{$l=l+1$}\ \Comment{increment $l$} 
		\EndWhile
		\If{${p}_{i} \neq 0$}\ \Comment{if ${p}_i$ still has some mass}
			\State{$t = t + {p}_i \cdot{C}_{i,\vect{s}[k]}$}\ \Comment{move the rest to ${q}_{\vect{s}[k]}$} 
		\EndIf
	\EndFor
	\State \Return $t$\ \Comment{return transportation cost $t$}
	\EndFunction
\end{algorithmic}
\caption{Approximate Computation of ICT}
\label{alg:aict}
\end{algorithm}
\vspace{-0.05in}

\section{Linear-Complexity Implementations} \label{sec:linear-complexity}

In this section, we focus on the ACT method because 1) it is a generalization of all the other methods presented, 
and 2) its complexity and accuracy can be controlled by setting the number $k$ of iterations  performed. 
We describe a data-parallel implementation of ACT, which achieves a linear time complexity when 
$k$ is a constant. Unlike the previous section, we do not assume that the cost matrix is given. We 
compute the transportation costs on the fly, and take into account the complexity of computing these costs as well.

A high-level view of the linear-complexity ACT algorithm (LC-ACT) is given in Figure~\ref{fig:lc-ict}. 
LC-ACT is strongly inspired by LC-RWMD. Just like LC-RWMD, it assumes that 
1) a query histogram is compared with a large number of database histograms, and 
2) the coordinate space is populated by the members of a fixed-size vocabulary. 
The complexity is reduced by eliminating the redundant and repetitive operations that 
arise when comparing one query histogram with a large number of database histograms. 

\begin{figure}[tb!]
  \centering%
  \includegraphics*[width=0.95\linewidth]{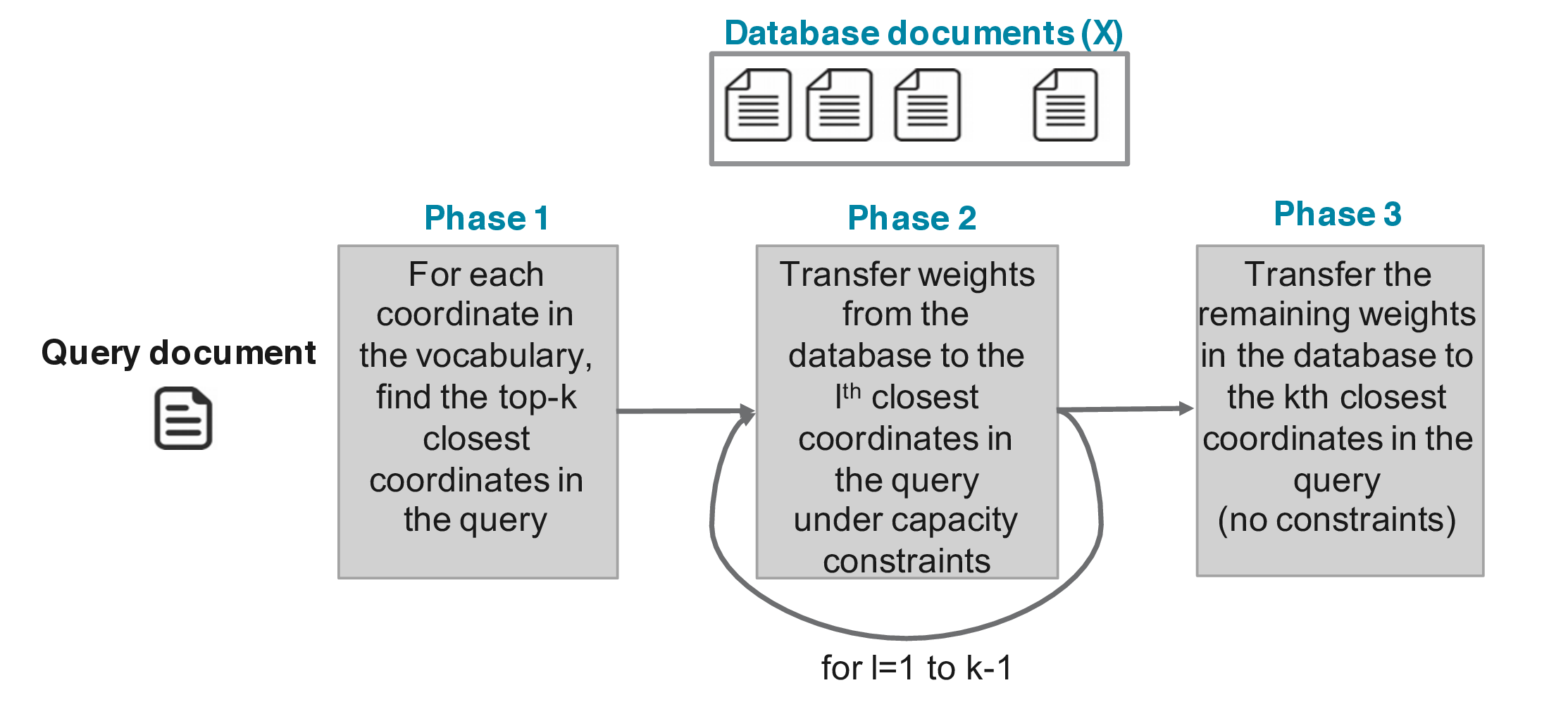}
  \vspace{-0.20in}
  \caption{Linear Complexity ACT}
  \vspace{-0.05in}
  \label{fig:lc-ict}
\end{figure}



Suppose that the dimension of the coordinates is $m$ and the size of the vocabulary is $v$. 
Let $\vect{V}$ be an $v \times m$ matrix that stores this information. Given a query histogram $\vect{q}$ of size $h$,
we construct a matrix $\vect{Q}$ of size $h \times m$ that stores the coordinates of the histogram
entries. Phase 1 of LC-ACT (see Fig.~\ref{fig:phase1}) performs a matrix-matrix multiplication 
between $\vect{V}$ and the transpose of $\vect{Q}$ to compute all pairwise distances between the coordinates of 
the vocabulary and the coordinates of the query. The result is a $v \times h$ distance matrix, denoted by $\vect{D}$.
As a next step, the top-$k$ smallest distances are computed in each row of $\vect{D}$. The result
is stored in a $v \times k$ matrix $\vect{Z}$. Furthermore, we store the indices of $\vect{q}$ that are associated
with the top-$k$
smallest distances in a $v \times k$ matrix $\vect{S}$. We can then construct another $v \times k$ 
matrix $\vect{W}$, which stores the corresponding weights of $\vect{q}$ by defining $\vect{W}_{i,l}=\vect{q}_{{S}_{i,l}}$ for $i=1, \ldots, v$ and $l=1, \ldots, k$.
The matrices $\vect{Z}$ and $\vect{W}$ are then used in Phase 2 to transport the largest possible
mass, which are constrained by $\vect{W}$, to the smallest possible distances, which are given by $\vect{Z}$.

The database histograms are stored in a matrix $\vect{X}$ (see Fig.~\ref{fig:phase2}), wherein each row 
stores one histogram. These histograms are typically sparse. Thus, the matrix $\vect{X}$ is 
stored using a sparse representation, e.g., in compressed sparse rows (csr) format. For simplicity, 
assume that $\vect{X}$ is stored in a dense format and $\vect{X}_{u,i}$ stores the weight of the $i$-th coordinate 
of the vocabulary in the $u$-th database histogram. Note that if the histograms have $h$ entries on average, 
the number of nonzeros of the matrix $\vect{X}$ would be equal to $nh$. 

\begin{figure}[tb!]
  \centering%
  \includegraphics*[width=0.9\linewidth]{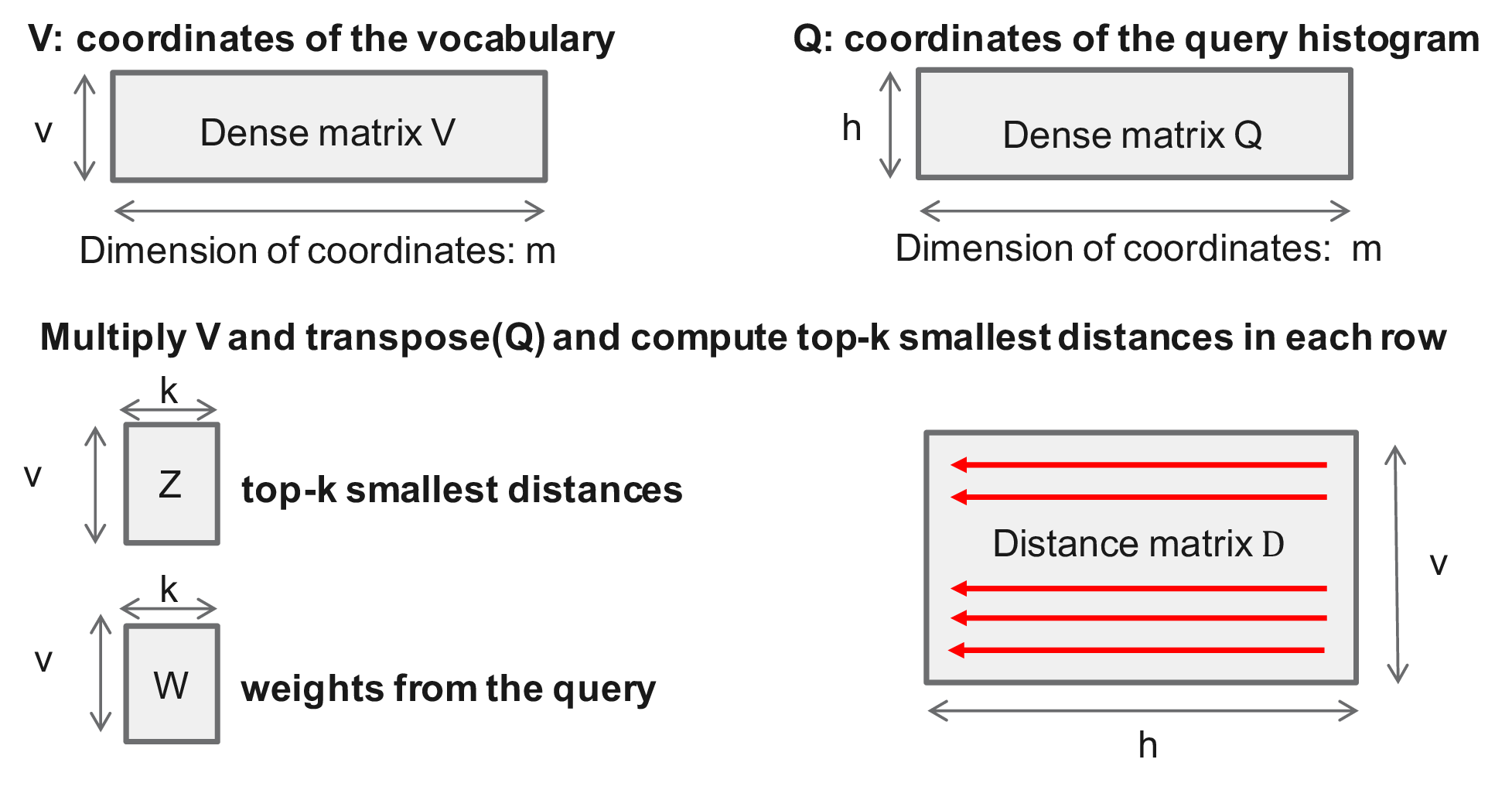}
  \vspace{-0.2in}
  \caption{Phase 1 of LC-ACT}
  \vspace{-0.1in}
  \label{fig:phase1}
\end{figure}

\begin{figure}[tb!]
  \centering%
  \includegraphics*[width=0.9\linewidth]{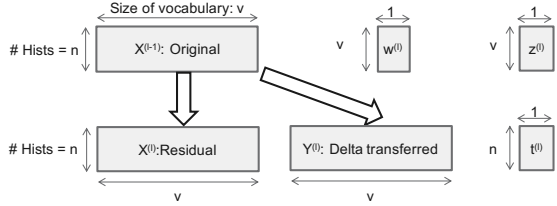}
  \vspace{-0.2in}
  \caption{Phase 2 of LC-ACT}
  \vspace{-0.1in}
  \label{fig:phase2}
\end{figure}

Phase 2 of ACT iterates the columns of $\vect{Z}$ and $\vect{W}$ and iteratively transfers weights from the database histograms $\vect{X}$ to the query histogram $q$. 
Let $\vect{X}^{(l)}$ represent the residual mass remaining in $\vect{X}$ after $l$ iterations, where $\vect{X}^{(0)}=\vect{X}$. 
Let $\vect{Y}^{(l)}$ store the amount of mass that is transferred from $\vect{X}^{(l-1)}$ in iteration $l$, which is  the difference between $\vect{X}^{(l-1)}$ and $\vect{X}^{(l)}$. 
Let $\vect{z}^{(l)}$ and $\vect{w}^{(l)}$ be the $l$-th columns of $\vect{Z}$ and $\vect{W}$, respectively; thus, 
$\vect{z}^{(l)}_{u}$ is the $l$-th smallest distance between the coordinate $u$ of the vocabulary and the coordinates of the query, 
and $\vect{w}^{(l)}_{u}$ is the respective weight of the query coordinate that produces the $l$-th smallest distance. 
The iteration $l$ of Phase 2 computes $\vect{Y}^{(l)}$ and $\vect{X}^{(l)}$:  
\vspace{-0.1in}
\begin{equation}
  \vect{Y}^{(l)}_{u,i} = \min_{
	\scriptsize 
        \begin{array}{c}
        u{\in}\{1 \ldots v\} \\
        i{\in}\{1 \ldots v\}
        \end{array}}  (\vect{X}^{(l-1)}_{u,i}, \vect{w}^{(l)}_{u}) . 
\label{eq:derivey}
\end{equation}
\vspace{-0.1in}
\begin{equation}
\vect{X}^{(l)} = \vect{X}^{(l-1)} - \vect{Y}^{(l)}. 
\label{eq:derivex}
\end{equation}
The cost of transporting $\vect{Y}^{(l)}$ to $q$ is given by $\vect{Y}^{(l)} \cdot \vect{z}^{(l)}$.
Let $\vect{t}^{(l)}$ be a vector of size $n$ that accumulates all the transportation 
costs incurred between iteration $1$ and iteration $l$:
\begin{equation}
\vect{t}^{(l)} = \vect{t}^{(l-1)} + \vect{Y}^{(l)} \cdot \vect{z}^{(l)}.
\label{eq:derivet1}
\end{equation}
After $k-1$ iterations of Phase 2, there might still be some mass remaining in $\vect{X}^{(l-1)}$.
Phase 3 approximates the cost of transporting the remaining mass to $q$ by multiplying $\vect{X}^{(l-1)}$ 
with $\vect{z}^{(k)}$. The overall transportation cost $\vect{t}^{(k)}$ is:
\begin{equation}
\vect{t}^{(k)} = \vect{t}^{(k-1)} + \vect{X}^{(l-1)} \cdot \vect{z}^{(k)}.
\label{eq:derivet2}
\end{equation}

The main building blocks of LC-ACT are matrix-matrix and matrix-vector multiplications, 
row-wise top-$k$ calculations, and parallel element-wise updates, all of which are data-parallel 
operations. Table~\ref{tab:complexity2} shows the complexity of computing LC-ACT between one 
query histogram and $n$ database histograms. Note that when $k$ is a constant, LC-ACT and LC-RWMD methods have the same complexity. 

\begin{table}[t!]
\caption{Complexity of LC-ACT ($n$ distances, $k$ iterations)}
  \vspace{-0.05in}
\label{tab:complexity2}
\begin{center}
\begin{tabular}{ll}
  \hline
  Time & Space \\ 
  \hline
  $O(vhm+nhk)$    & $O(nh+vm+vh+vk)$ \\
  \hline
  \vspace{-0.35in}
\end{tabular}
\end{center}
\end{table}



\section{Evaluation}

We performed experiments on two public datasets: \emph{20 Newsgroups} 
is a text database of newsgroup documents, partitioned evenly across 20 different classes\footnote{http://qwone.com/~jason/20Newsgroups/}, 
and \emph{MNIST} is an image database of 
greyscale hand-written digits that are partitioned evenly across 10 classes\footnote{http://yann.lecun.com/exdb/mnist/}. 
The MNIST images are mapped to histograms as illustrated in Fig.~\ref{fig:illustration}, 
wherein the weights are normalized pixel values and the embedding vectors indicate the coordinates of the pixels.
The words in 20 Newsgroups documents are mapped to a 300-dimensional real-valued embedding 
space using Word2Vec vectors that are pre-trained on Google News\footnote{https://code.google.com/archive/p/word2vec/}, 
for which the size of the vocabulary ($v$) is 3M words and phrases. 
In our setup, each histogram bin is associated with a word from this vocabulary. 
The first 100 words of the vocabulary are treated as stop words. These stop words 
and the Word2Vec phrases are not mapped to histogram bins.
The histogram weights indicate normalized frequencies of words found in each document. 
In addition, 20 Newsgroups histograms are truncated to store only the most-frequent 500 words found in each document.
Tab.~\ref{tab:benchmarks} shows some properties of the datasets and the results of our preprocessing.
Note that the size of the vocabulary used has an impact on the complexity of our methods (see Tab.~\ref{tab:complexity2}).

The results we provide in the remainder of the text are associated 
with linear complexity implementations of RWMD, OMR, and ACT methods.
To improve the robustness of these methods, we compute two asymmetric 
lower bounds and take the maximum of the two as discussed in Section~\ref{sec:rwmd}. 
Word2Vec embedding vectors are $L_2$-normalized, but MNIST embedding 
vectors are not normalized in our setup. In addition, when computing 
our approximations, the histogram weights are always 
$L_1$-normalized. Lastly, the transportation cost between two words or 
pixels is the Euclidean ($L_2$) distance between their embedding vectors.

\begin{table}[t!]
\vspace{-0.1in}
\caption{Dataset properties: no. docs ($n$), average size of histograms ($h$), original vocabulary size ($v$), size of used vocabulary}
\label{tab:benchmarks}
\begin{center}
\begin{tabular}{lcccc}
  \hline
                 & Size $n$ & Average $h$ & Original $v$ & Used $v$ \\ 
  \hline
  20 News        & 18828  & 78.8  & 3M  & 69682  \\
  MNIST          & 60000  & 149.9 & 784 & 717 \\
  \hline
  \vspace{-0.3in}
\end{tabular}
\end{center}
\end{table}

In our experiments, we treated each document of the database as a query and compared it with 
every other document in the database. Based on the distance measure used in the 
comparison, for each query document, we identified the top-$\ell$ nearest neighbors in the 
database. After that, for each query document, we computed the percentage of documents 
in its nearest-neighbors list that have the same label.
We averaged this metric over all the query documents and computed it as a function 
of $\ell$. The result is the average \emph{precision @ top-$\ell$} for the query documents,
and indicates the expected accuracy of nearest neighbors search queries.

We compared our new distance measures with simple baselines, such as, Bag-of-Words (BoW) 
cosine similarity and the Word Centroid Distance (WCD)~\cite{kusnerskw15} measures, both 
of which exhibit a lower algorithmic complexity than our methods. The BoW approach does not 
use the proximity information provided by the embedding vectors. It simply computes a dot 
product between two sparse histograms after an $L_2$ normalization of the weights. 
The complexity of computing BoW cosine similarity between one query histogram and $n$ 
database histograms is $O(nh)$. The WCD measure, on the other hand, is closely related 
to document embedding techniques. For each document, it first computes a centroid vector of 
size $m$, which is a weighted average of the embedding vectors, and then determines  
the Euclidean distances between these centroid vectors. The complexity of 
computing WCD between one query document and $n$ database documents is $O(nm)$.  


We compared our distance measures with state-of-the-art EMD approximations as well, 
such as WMD and Sinkhorn distance. WMD uses the FastEMD library\footnote{https://github.com/LeeKamentsky/pyemd} to approximate the EMD,
which uses a thresholding technique~\cite{Pele2009} to reduce the time to compute EMD. 
In addition, WMD uses an RWMD-based pruning technique to reduce the number of calls to FastEMD~\cite{kusnerskw15}.
To improve the WMD performance further, we developed a multi-threaded CPU implementation of the pruning technique.
We also compared our methods with Cuturi's open-source implementation
of Sinkhorn's algorithm\footnote{http://marcocuturi.net/SI.html}, which can be executed on both CPUs and GPUs. 
We used $\lambda=20$ as the entropic regularization parameter because it offered 
a good trade-off between the accuracy and the speed of Sinkhorn's algorithm.

We have developed GPU-accelerated implementations of the WCD, RWMD, OMR, ACT, and BoW cosine similarity 
methods and evaluated their performance on an NVIDIA\textsuperscript{\textregistered} GTX 1080Ti GPU. 
Cuturi's Sinkhorn implementation has also been executed on the same GPU.  
Our multithreaded WMD implementation has been deployed on an 8-core Intel\textsuperscript{\textregistered} i7-6900K CPU. 
Top-$\ell$ calculations have been performed on the same Intel\textsuperscript{\textregistered} i7-6900K CPU in all cases.

Theorem~\ref{theorem:ordering} states that the more complex the considered algorithms, 
the smaller the gap to the EMD and, hence, the better the accuracy. The least complex 
ACT algorithm is the RWMD, which corresponds to the ACT-0 with zero iterations 
in Phase 2 (see Fig.~\ref{fig:lc-ict}). The second most complex algorithm is the OMR. 
The third most complex ACT algorithm is ACT-1 with a single iteration in Phase 2.
The most complex ACT algorithm we considered was ACT-15 with 15 iterations in Phase 2. 
Our experiments show that the search accuracy improves with the complexity and, 
thus, illustrate the accuracy vs complexity trade-off. Typically, most of the 
improvement in the search accuracy is achieved by the first iteration of Phase 2,
and subsequent iterations result in a limited improvement only. As a result, 
ACT-1 offers very favorable accuracy and runtime combinations. 


Figure~\ref{fig:combined} (a) shows the accuracy and runtime trade-offs between 
different methods on the 20 Newsgroups dataset. Note that even though ACT-1 is 
approximately 20-fold slower than BoW cosine similarity, it offers a 4.5\% to 7.5\% higher search accuracy.
Typically, the accuracy improvement with respect to BoW becomes larger as we increase $\ell$. 
Fig.~\ref{fig:combined} (a) also shows that ACT-1 is approximately 20000-fold faster than WMD, 
but offers a similar search accuracy. It is only 30\% slower than RWMD, but results in a 2\% to 3.5\% 
higher search accuracy. OMR is somewhere between RWMD and ACT-1 in terms of both runtime and search accuracy.
Finally, ACT-7 is approximately 10000-fold faster than WMD, and offers a slightly higher search accuracy! 

\begin{figure}[tb!]
  \centering%
  \vspace{-0.03in}
  \includegraphics*[width=1.0\linewidth]{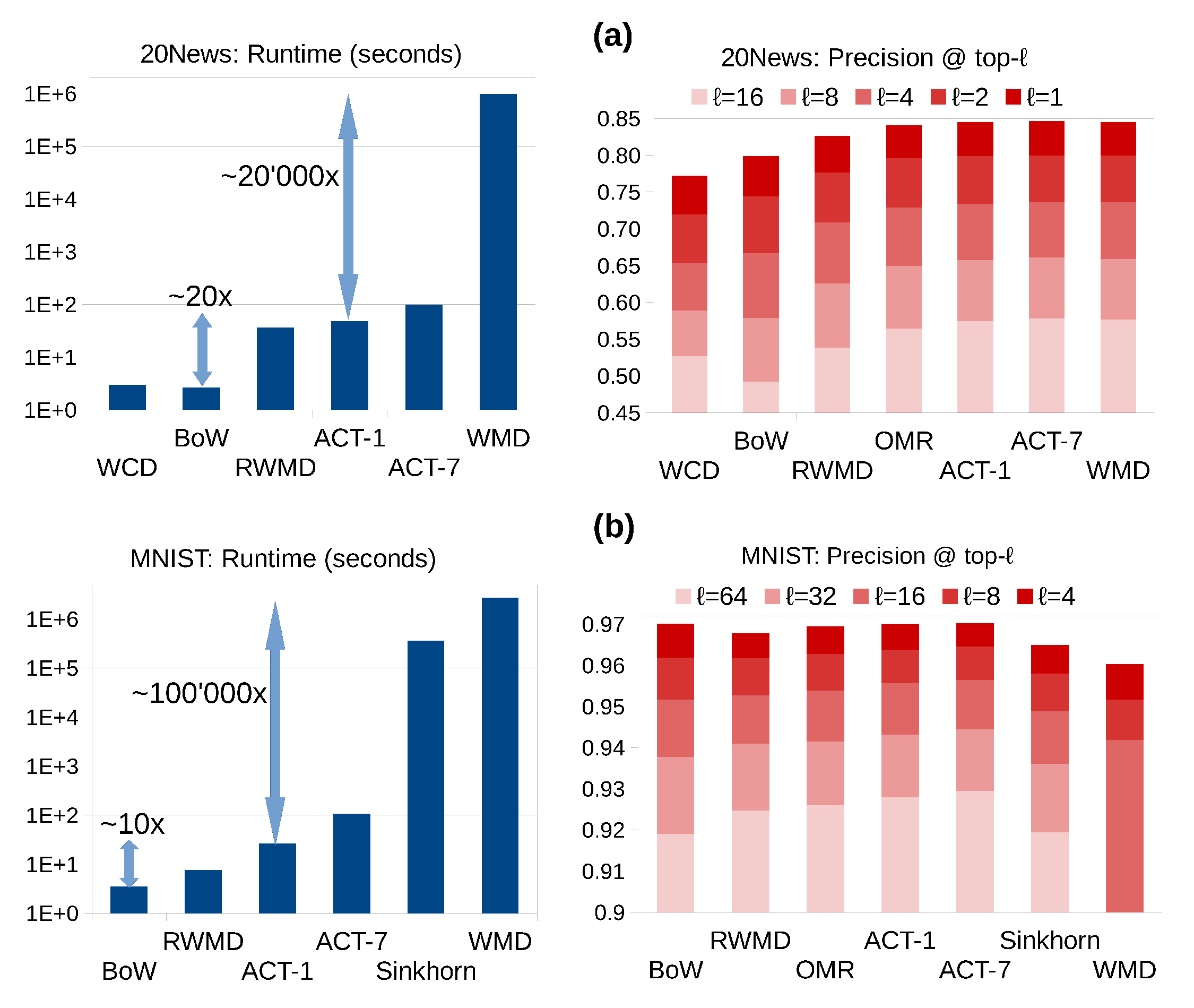}
  \vspace{-0.3in}
  \caption{Runtime vs accuracy for 20News and the MNIST subset}
  \vspace{-0.17in}
  \label{fig:combined}
\end{figure}

In case of MNIST, because the number of dimensions is small 
($m=2$), RWMD is almost as fast as BoW cosine similarity. However, 
the runtime of the Phase 2 of the ACT-1 method is much more significant 
than that of its Phase 1. Therefore, the runtime increase with respect 
to BoW is around ten fold for ACT-1. Nevertheless, when using ACT-1, computing
all pairwise distances between 60000 MNIST training images (i.e., 3.6 billion 
distance computations) takes only 3.3 minutes. The accuracy comparisons between
BoW, RWMD, and ACT methods for the complete MNIST database are given in Tab.~\ref{tab:mnist_wo_bg}.
The accuracy is already very high when using BoW because the images are normalized and centered. 
Our methods are comparable to BoW for small $\ell$, but outperform it for large enough $\ell$.

Computing all pairwise distances for the complete set of MNIST training images 
would take months when using WMD and Sinkhorn's algorithm. To enable comparisons 
with these two methods, we have set up a simpler experiment. We used only 
the first 6000 MNIST training images as our query documents and compared them 
with the complete set of 60000 MNIST training images. The results are given 
in Fig.~\ref{fig:combined} (b). We observe that ACT-1 is four orders of magnitude 
faster than Sinkhorn's algorithm when running on the same GPU, yet it achieves a higher search accuracy! 
Similarly, ACT-1 is five orders of magnitude faster than WMD while achieving a higher search accuracy! 
These results show that the OMR and the ACT measures we proposed are meaningful on their own, and using more 
complex measures, such as WMD or Sinkhorn does not necessarily improve the accuracy of nearest-neighbors search.



Table~\ref{tab:mnist} illustrates the sensitivity of RWMD to a minor 
change in the data representation. Here, we simply explore the impact 
of including the background (i.e. the black pixels) in the MNIST 
histograms. The most immediate result is that when comparing two 
histograms, all their coordinates overlap. As a result, the distance 
computed between the histograms by RWMD is always equal to zero, and 
the top-$\ell$ nearest neighbors are randomly selected, resulting in a 
precision of 10\% for RWMD. The OMR technique solves this problem 
immediately even though its accuracy is lower than that of BoW cosine 
similarity. In fact, several iterations of ACT are required to 
outperform BoW. However, these results demonstrate
the improved robustness and effectiveness of our methods in comparison to RWMD.

\begin{table}[t!]
\caption{Precision @ top-$\ell$ for MNIST (without background)}
\label{tab:mnist_wo_bg}
\begin{center}
\begin{tabular}{lccccc}
  \hline
  $\ell$      & BoW & RWMD & ACT-1 & ACT-3 & ACT-7 \\ 
  \hline
  1           & 0.9771 & 0.9752 & 0.9776 & 0.9780 & 0.9781 \\
  16          & 0.9480 & 0.9481 & 0.9510 & 0.9520 & 0.9521 \\
  128         & 0.8874 & 0.8963 & 0.8997 & 0.9014 & 0.9016 \\
  \hline
  \vspace{-0.4in}
\end{tabular}
\end{center}
\end{table}

\begin{table}[t!]
\caption{Precision @ top-$\ell$ for MNIST (with background)}
\label{tab:mnist}
\begin{center}
\begin{tabular}{lccccc}
  \hline
  $\ell$      & BoW & RWMD & OMR & ACT-7 & ACT-15 \\ 
  \hline
  1           & 0.9771 & 0.1123 & 0.9707 & 0.9756 & 0.9783\\
  16          & 0.9480 & 0.1002 & 0.9368 & 0.9470 & 0.9520\\
  128         & 0.8874 & 0.1002 & 0.8692 & 0.8872 & 0.8999\\
  \hline
  \vspace{-0.4in}
\end{tabular}
\end{center}
\end{table}

\section{Conclusions}

This paper offers new theoretical and practical results for improving 
the efficiency and the accuracy of approximate EMD computation in both
high and low dimensions. 
We identify the shortcomings of the RWMD measure and propose improved lower 
bounds that result in a higher nearest-neighbors-search accuracy and robustness without 
increasing the computational complexity significantly. Under realistic assumptions, the 
complexity of our methods scale linearly in the size of the input probability distributions. 
In addition, our methods are data-parallel and well-suited for GPU acceleration. 
The experiments demonstrate a four orders of magnitude improvement 
of the performance without any loss of nearest-neighbors-search accuracy 
with respect to WMD on high-dimensional text datasets. 
Similar improvements have been achieved with 
respect to Sinkhorn's algorithm on two-dimensional image datasets.


\section*{Acknowledgements}
We would like to thank Celestine D\"{u}nner, Haralampos Pozidis, Ahmet Solak, and Slavisa Sarafijanovic from IBM Research -- Zurich, Marco Cuturi from Google Brain and Institut Polytechnique de Paris, and the anonymous reviewers of ICML 2019 Conference for their valuable comments.


\bibliographystyle{icml2019}

\pagebreak
\onecolumn

\pagebreak

\appendix

\section{Optimality and Effectiveness} \label{sec:optimality}


Alg.~\ref{alg:ict} computes an optimum flow $\vect F^*$, whose components are determined by the quantities $r$ in step 4. Namely, the components of the $i$-th row of $\vect F^*$, are given recursively as $F^*_{i,\vect s[1]} = \min(p_i, q_{\vect s[1]})$ and $F^*_{i,\vect s[l]} = \min(p_i - \sum_{u=1}^{l-1} F^*_{i,\vect s[u]}, q_{\vect s[l]})$ for $l=2, \ldots, h_q$.

\begin{lemma} \label{lemma:ict}
  Each row $i$ of the flow $\vect F^*$ of Algorithm~\ref{alg:ict} has a certain number $k_i$, $1\leq k_i\leq h_q$ of nonzero components, which are given by
  $F^*_{i,\vect s[l]} = q_{\vect s[l]}$ for $l=1,\ldots, k_i-1$ and
  $F^*_{i,\vect s[k_i]} = p_i - \sum_{l=1}^{k_i-1}q_{\vect s[l]}$. 
\end{lemma}
The Lemma follows by keeping track of the values of the term $r$ in step~4 in Alg.~\ref{alg:ict}.
An immediate implication is that the flow $F^*$ satisfies the constraints (\ref{eq:outflow}) and  (\ref{eq:relaxed}). 
One can also show that $F^*$ is a minimal solution of (\ref{eq:objective}) under the constraints (\ref{eq:outflow}) and  (\ref{eq:relaxed}),
and this leads to the following theorem.

\begin{theorem}
  (i) The flow $F^*$ of Algorithm~\ref{alg:ict} is an optimal solution of the relaxed minimization problem given by (\ref{eq:objective}), (\ref{eq:outflow}) and (\ref{eq:relaxed}).
  (ii) ICT provides a lower bound on EMD.
\end{theorem}

\begin{proof}

Proof of part (i): It has already been shown that the flow $\vect F^*$ satisfies constraints (\ref{eq:outflow}) and (\ref{eq:relaxed}), and it remains to show that $\vect F^*$ achieves the minimum in (\ref{eq:objective}). To this end, let  $\vect F$ be any nonnegative flow, which satisfies (\ref{eq:outflow}) and (\ref{eq:relaxed}). To show that  $\vect F^*$ achieves the minimum in (\ref{eq:relaxed}), it is enough to show that for every row $i$, one has
$\sum_{j} F_{i,j} C_{i,j} \geq \sum_{j} F^*_{i,j} C_{i,j}$, which then implies $\sum_{i,j} F_{i,j} C_{i,j} \geq \sum_{i,j} F^*_{i,j} C_{i,j}$.

By Alg.~\ref{alg:ict}, there is a reordering given by the list $\vect s$ such that
\begin{equation} \label{eq:proof1}
C_{i,\vect s[1]}\leq C_{i,\vect s[2]}\leq \ldots \leq C_{i,\vect s[n_q]}.
\end{equation}
By Lemma~\ref{lemma:ict}, there is a $k_i\leq n_q$ such that $\sum_{l=1}^{k_i} F^*_{i,\vect s[l]} = p_i$ and  $F^*_{i,\vect s[l]} = 0$ for $l>k_i$. Furthermore by Lemma~\ref{lemma:ict} and by constraint (\ref{eq:relaxed}) on $\vect F$ , it follows that
\begin{equation} \label{eq:proof2}
F_{i,\vect s[l]} \leq q_{\vect s[l]} = F^*_{i,\vect s[l]} \quad \text{for $l=1, \ldots, k_i-1$}.
\end{equation}
The outflow-constraint (\ref{eq:outflow}) implies
$\sum_{j} F_{i,j} = p_i = \sum_{j} F^*_{i,j}$ or, equivalently,
\begin{equation} \label{eq:proof3}
\sum_{l=k_i}^{n_q} F_{i,\vect s[l]} =
                   F^*_{i,\vect s[k_i]} + \sum_{l=1}^{k_i-1}(F^*_{i,\vect s[l]} - F_{i,\vect s[l]}).
\end{equation}

In the following chain of inequalities, the first inequality follows from (\ref{eq:proof1}), and  (\ref{eq:proof3}) implies the equality in the second step.
\begin{eqnarray*}
\sum_{l=k_i}^{n_q} C_{i,\vect s[l]}F_{i,\vect s[l]}
&{\geq}&C_{i,\vect s[k_i]}\sum_{l=k_i}^{n_q} F_{i,\vect s[l]} \\
  &{=}&C_{i,\vect s[k_i]}(F^*_{i,\vect s[k_i]} + \sum_{l=1}^{k_i-1}(F^*_{i,\vect s[l]} - F_{i,\vect s[l]})) \\
  &{=}&C_{i,\vect s[k_i]}F^*_{i,\vect s[k_i]} + \sum_{l=1}^{k_i-1}C_{i,\vect s[k_i]}(F^*_{i,\vect s[l]} - F_{i,\vect s[l]}) \\
  &{\geq}&C_{i,\vect s[k_i]}F^*_{i,\vect s[k_i]} + \sum_{l=1}^{k_i-1}C_{i,\vect s[l]}(F^*_{i,\vect s[l]} - F_{i,\vect s[l]}).
\end{eqnarray*}
The inequality in the last step follows from (\ref{eq:proof1}) and the fact that the terms
$F^*_{i,\vect s[l]} - F_{i,\vect s[l]}$ are nonnegative by (\ref{eq:proof2}). By rewriting the last inequality, one obtains the desired inequality
\begin{eqnarray*}
\sum_{j} F_{i,j} C_{i,j} &=& \sum_{l=1}^{n_q} F_{i,\vect s[l]} C_{i,\vect s[l]} \\
&\geq& \sum_{l=1}^{k_i} F^*_{i,\vect s[l]} C_{i,\vect s[l]} \\
&=& \sum_{j} F^*_{i,j} C_{i,j} , 
\end{eqnarray*}
where in the last equation  $F^*_{i,\vect s[l]} = 0$ for $l>k_i$ is used.

Proof of part (ii): Since ICT is a relaxation of the constrained minimization problem of the EMD, ICT provides a lower bound on EMD given by the output of Alg.~\ref{alg:ict}, namely, $\sum_{i,j} F^*_{i,j} C_{i,j} = \text{ICT}(\vect p, \vect q) \leq \text{EMD}(\vect p, \vect q)$.

\end{proof}

Similar to Alg.~\ref{alg:ict}, Alg.~\ref{alg:aict} also determines an optimum flow $F^*$, which now depends on the number of iterations $k$. 
\begin{lemma} \label{lemma:aict}
  Each row $i$ of the flow $\vect F^*$ of Algorithm~\ref{alg:aict} has a certain number $k_i$, $1\leq k_i\leq k$ of nonzero components, which are given by
  $F^*_{i,\vect s[l]} = q_{\vect s[l]}$ for $l=1,\ldots, k_i-1$ and
  $F^*_{i,\vect s[k_i]} = p_i - \sum_{l=1}^{k_i-1}q_{\vect s[l]}$. 
\end{lemma}
Based on this Lemma, one can show that the flow $F^*$ from Algorithm~\ref{alg:aict} is an optimum solution to the minimization problem given by (\ref{eq:objective}), (\ref{eq:outflow}) and (\ref{eq:relaxed}), in which the constraint (\ref{eq:relaxed}) is further relaxed in function of the predetermined parameter $k$. Since the constrained minimization problems for ICT, ACT, OMR, RWMD form a chain of increased relaxations of EMD, one obtains the following result. 

\begin{theorem} 
  For two normalized histograms $\vect p$ and $\vect q$:
  $\text{RWMD}(\vect{p},\vect{q}) \leq  \text{OMR}(\vect{p},\vect{q}) \leq \text{ACT}(\vect{p},\vect{q}) \leq \text{ICT}(\vect{p},\vect{q}) \leq \text{EMD}(\vect{p},\vect{q})$.
\end{theorem}


We call a nonnegative cost function $\vect C$ {\it effective}, if for any indices
$i, j$, the equality $C_{i,j}= 0$ implies $i = j$. For a topological
space, this condition is related to the Hausdorff property. For
an effective cost function $\vect C$, one has $C_{i,j} > 0$ for all $i \neq j$,
and, in this case, $\text{OMR}(\vect p, \vect q) = \sum_{i,j}C_{i,j}F^*_{i,j}=0$ implies $F^*_{i,j}=0$ for $i \neq j$ and, thus, $k_i=1$ in
Lemma~\ref{lemma:aict} and, thus, $\vect F^*$ is diagonal with $F^*_{i,i}=p_i$. This implies $p_i\leq q_i$ for all $i$ and, since both histograms are normalized, one must have $\vect p = \vect q$.
\begin{theorem}
  If the cost function $\vect C$ is effective, then  $\text{OMR}(\vect p, \vect q) = 0$ implies $\vect p = \vect q$, i.e., OMR is effective.
\end{theorem}

\begin{remark}
If  $\text{OMR}$ is effective, then, a fortiori, ACT and ICT are also effective. However, RWMD does not share this property.
\end{remark}

\section{Complexity Analysis} \label{sec:complexity}


The algorithms presented in Section 3 assume that the cost matrix $\vect{C}$ is given, 
yet they still have a quadratic time complexity in the size of the histograms. 
Assume that the histograms size is $h$. Then, the size of $\vect{C}$ is $h^2$.
The complexity is determined by the row-wise reduction operations on $\vect{C}$. 
In case of the OMR method, the top-$2$ smallest values are computed in each row of $\vect{C}$ and a 
maximum of two updates are performed on each bin of $\vect{p}$. Therefore, the complexity is $O(h^2)$. 
In case of the ACT method, the top-$k$ smallest values are computed in each row, and  
up to $k$ updates are performed on each histogram bin. Therefore, the complexity is $O(h^2\log{k}+hk)$.
The ICT method is the most expensive one because 1) it fully sorts the rows of $\vect{C}$, and 2) it 
requires $O(h)$ iterations in the worst case. Its complexity is given by $O(h^2\log{h})$.


In Section~\ref{sec:linear-complexity}, the complexity of Phase 1 of the LC-ACT algorithm is $O(vhm + nh\log{k})$ 
because the complexity of the matrix multiplication that computes $\vect{D}$ is $O(vhm)$, and the complexity of 
computing top-$k$ smallest distances in each row of $\vect{D}$ is $O(nh\log{k})$. The complexity of performing (\ref{eq:derivey}), 
(\ref{eq:derivex}), (\ref{eq:derivet1}), and (\ref{eq:derivet2}) are $O(nh)$ each. When $k-1$ 
iterations of Phase 2 is applied, the overall time complexity of the LC-ACT algorithm is $O(vhm + knh)$.
Note that when the number of iterations $k$ performed by LC-ACT is a constant, LC-ACT and LC-RWMD 
have the same time complexity. When the number of iterations are in the order of the dimensionality 
of the coordinates (i.e., $O(k)=O(m)$) and the database is sufficiently large (i.e., $O(n)=O(v)$), LC-ACT 
and LC-RWMD again have the same time complexity, which increases linearly in the size of the histograms $h$. 
In addition, the sizes of the matrices $\vect{X}$, $\vect{V}$, $\vect{D}$, and $\vect{Z}$ are $nh$, $vm$, $vh$, and $vk$, respectively. 
Therefore, the overall space complexity of the LC-ACT algorithm is $O(nh+vm+vh+vk)$.

\end{document}